\documentclass[journal]{IEEEtran}
\pdfoutput=1

\usepackage{cite}
\usepackage[numbers]{natbib}
\usepackage{multicol}
\usepackage[bookmarks=true]{hyperref}
\usepackage{graphicx}
\usepackage{amsmath} 
\usepackage{amssymb}  
\usepackage{amsthm}
\usepackage{array}
\usepackage{algorithm}
\usepackage{algpseudocode}
\usepackage{color}

\newtheorem{theorem}{Theorem}
\newtheorem{defn}{Definition}
\newtheorem{corollary}{Corollary}

\algtext*{EndWhile}
\algtext*{EndIf}
\algtext*{EndFor}

\graphicspath{{./}{figures/}}

\begin{document}

\title{Asymptotically Optimal Planning by Feasible Kinodynamic Planning in State-Cost Space}

\author{Kris~Hauser
	and Yilun~Zhou
\thanks{K. Hauser and Yilun Zhou are with the Departments
of Electrical and Computer Engineering and of Mechanical Engineering and Materials Science, Duke University, Durham,
NC, 27708 USA e-mail: \{kris.hauser,yilun.zhou\}@duke.edu.}
}

\maketitle

\begin{abstract}
This paper presents an equivalence between feasible kinodynamic planning and optimal kinodynamic planning, in that any optimal planning problem can be transformed into a series of feasible planning problems in a state-cost space whose solutions approach the optimum.  This transformation gives rise to a meta-algorithm that produces an asymptotically optimal planner, given any feasible kinodynamic planner as a subroutine.  The meta-algorithm is proven to be asymptotically optimal, and a formula is derived relating expected running time and solution suboptimality. It is directly applicable to a wide range of optimal planning problems because it does not resort to the use of steering functions or numerical boundary-value problem solvers. On a set of benchmark problems, it is demonstrated to perform, using the EST and RRT algorithms as subroutines, at a superior or comparable level to related planners. 
\end{abstract}

\section{Introduction}
Optimal motion planning is a highly active research topic in robotics, due to the pervasive need to compute paths that simultaneously avoid complex obstacles, satisfy dynamic constraints, and are high quality according to some cost function.  Recent advances in sampling-based optimal motion planning build on decades of work in the topic of {\em feasible} motion planning, in which costs are ignored.  However, the field is still some ways away from general-purpose optimal planning algorithms that accept arbitrary black-box constraints and costs as input.  In particular, {\em optimality under kinematic and differential constraints} remains a major challenge for sampling-based planners.

This paper presents a new {\em state-cost space} formulation that transforms optimal motion planning problems into feasible kinodynamic (both kinematically- and differentially-constrained) motion planning problems.  Using this formulation, we introduce a meta-algorithm, AO-$x$, to adapt any feasible kinodynamic planner $x$ into an asymptotically-optimal motion planner, provided that $x$ satisfies some relatively unrestrictive conditions, e.g., expected running time is finite.  The meta-algorithm accepts arbitrary cost functions, including non-differentiable ones, and handles whatever kinematic and differential constraints are handled by the underlying feasible planner. 

\begin{figure}
\centering
\includegraphics[width=0.45\linewidth]{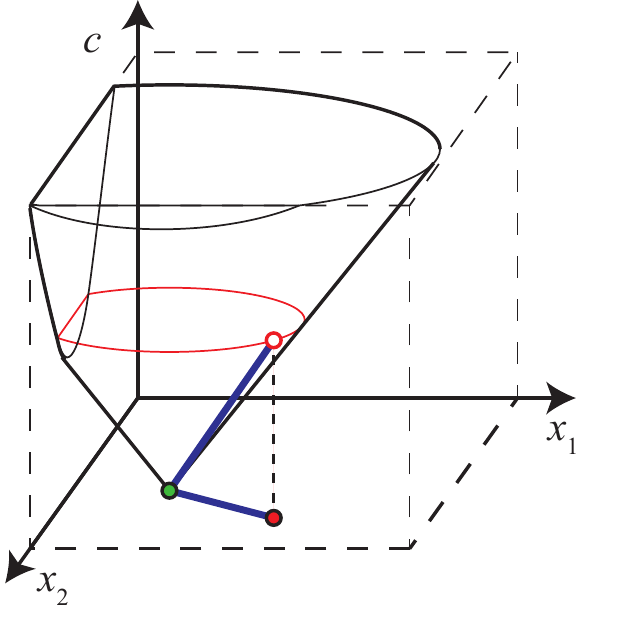} 
\includegraphics[width=0.45\linewidth]{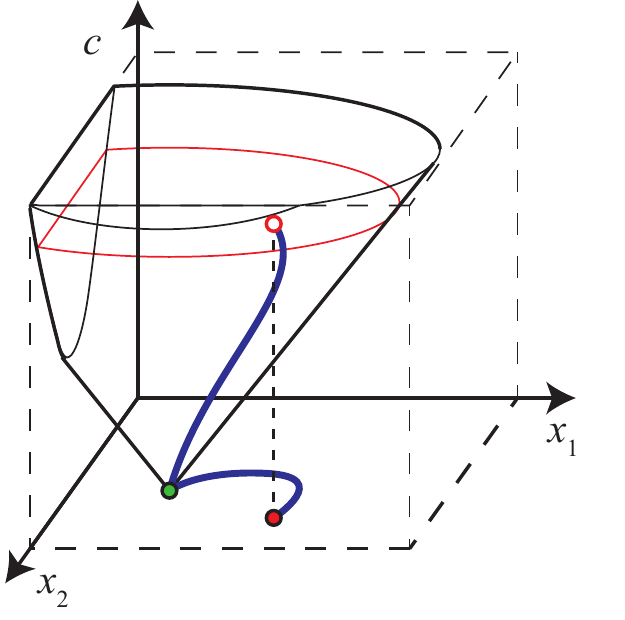} 
\caption{The state-cost space for a 2D, kinematically constrained problem with a path length cost function.  State-cost space is 3D, with a conical reachable set with an apex at the start configuration (green circle).  Two paths to the same state-space target (red filled circles) follow trajectories that arrive at different points in state-cost space (red open circles).}
\label{fig:StateCostSpace}
\end{figure}

The formulation is rather straightforward: $n$-dimensional state is augmented with an auxiliary cost variable, which measures the cost-to-come (i.e., accumulated cost from the start state), yielding a $(n+1)$-dimensional, dynamically-constrained feasible problem in (state, cost) space (Fig.~\ref{fig:StateCostSpace}).  The meta-algorithm proceeds by generating a series of feasible trajectories in state-cost space with progressively lower costs.  This is accomplished by first generating a feasible trajectory in state space, then progressively shrinking an upper bound on cost according to the cost of the best path found so far.  This meta-algorithm is proven to converge toward an optimal path under relatively unrestrictive conditions. 

The AO-$x$ meta-algorithm is demonstrated on practical examples using the RRT~\cite{rrt-LaValle-Kuffner} and EST~\cite{HKLR02} algorithms as subroutines for feasible kinodynamic planning.  Due to prior theoretical work on the running time of EST, we are able to prove that the expected running time of the meta-algorithm is $O(\epsilon^{-2} \ln \ln \epsilon^{-1} )$, where $\epsilon$ is the solution suboptimality.
Critically, this is one of the few asymptotically-optimal planners that exclusively uses control-sampling to handle dynamic constraints, rather than resorting to a steering function or a numerical two-point boundary value problem solver.  The new method outperforms prior planners in several toy scenarios including both dynamic constraints and complex cost functions.

\section{Background and Related Work}

Optimal motion planning has been a topic of renewed activity in robotics largely due to the advent of sampling-based motion planners that are proven to be asymptotically optimal~\cite{karaman2011sampling}.  But the community has had a long history of interest in optimal motions.
Numerical trajectory optimization techniques~\cite{betts1998survey,SimulatedAnnealing,BMS+01,CHOMP} have been long studied, but have several drawbacks.  First, they are prone to falling into local minima, and second they typically require differentiable constraint and cost representations, which are often hard to produce for complex obstacles.  Grid-based planners~\cite{ferguson2006using,Likhachev03b,FMM-Sethian} are often fast in low dimensional spaces but suffer from the ``curse of dimensionality,'' with performance degrading rapidly in spaces of higher dimension~\cite{LH2014}.  Sampling-based planners were originally developed to overcome many of these challenges, and have been shown to have excellent empirical performance in finding feasible paths in high-dimensional spaces, both without~\cite{KavrakiPRM} and with dynamic constraints~\cite{HKLR02,rrt-LaValle-Kuffner}.  However, they tend to produce jerky paths that are far from optimal.  Some hybrid approaches have combined sampling-based planning with local optimization to produce better paths~\cite{anytime2013,Vou2005}. 

More recently, sampling-based optimal planners like RRT* produce asymptotically-optimal paths whose costs converge in expectation toward the optimal path.  The insight is that optimal paths can be obtained by judiciously ``rewiring'' a tree of states to add connections that reduce the cost to a given node in the tree.  However, this requires a {\em steering function}, a method to produce a curve between two states that is optimal when obstacles are ignored.  In systems without dynamic constraints, this is as simple as generating a straight line.  But steering functions for dynamically-constrained systems are much harder to come by.

Several authors have extended RRT* to dynamically-constrained systems.  It is relatively easy to apply RRT* to dynamically-constrained systems if a steering function is available~\cite{KF2010}.  Proving convergence is harder, requiring analysis of small-time controllability conditions~\cite{KF2013}.  Other authors have extended RRT* to systems whose dynamics and costs are (or can be approximated) by linear and quadratic functions, respectively, by definition of a suitable steering function based on the LQR principle~\cite{perez2012lqr,WB2013}.  When more complex differential constraints are involved, it may not be possible to devise a suitable steering functions.  One method generated complex maneuvers using RRT* and performed each rewiring step by numerically solving a two-point boundary value problems~\cite{JKF2011}.  This adds greatly to computational expense.  A similar method performed rewiring using a spline-based trajectory representation that is optimized via a nonlinear program solver~\cite{SL2014}. 

The prior work with closest relation to ours in terms of generality of applicability is the Sparse-Stable-RRT planner~\cite{LLB2013}. Like our work, it avoids the use of a steering function entirely and samples directly in control space. Approximate rewiring is performed by allowing connections to points that are ``near enough'' according to a state-space distance metric.  This scheme was proven to satisfy asymptotic near-optimality, which is the property of converging toward a path with bounded suboptimality~\cite{6696508}.  In a more recent paper, the same authors have extended it to an asymptotically-optimal planner, SST*, by progressively shrinking nearness threshold parameters~\cite{LLB2014}.  However, Sparse-Stable-RRT and SST* and have many parameters to tune, and our experiments suggest that AO-$x$ planners in general outperform both planners.

We note the similarity of algorithm to Anytime-RRT~\cite{Ferguson2006anytime}, except that the planner uses state-cost space rather than simply state space (and has fewer parameters to tune).  Hence, Anytime-RRT does not obey any theoretical asymptotic- or near-optimality guarantees.  This is critical in practice, as our experiments suggest Anytime-RRT tends not to converge to an optimum.


\section{Theoretical Formulation}

This section presents the state-cost space formulation, the meta-algorithms, and theoretical results regarding asymptotically-optimality.

\subsection{Terminology}
First we define key concepts of feasible, optimal, and boundedly-suboptimal planning problems, as well as complete, probabilistically complete, and asymptotically optimal planners.  Let $X$ denote the state space.

\begin{defn}
\label{def:Kinodynamic}
A {\em feasible} (kinodynamic) planning problem $P=(X,U,x_I,G,F,B,D)$ asks to produce a trajectory $y(s):[0,S]\rightarrow X$ and control $u(s):[0,S]\rightarrow U$ such that:
\begin{alignat}{2}
y(0)&=x_I && \text{ (initial state) } \\
y(1)&\in G\subseteq X && \text{ (goal state)} \\
y(s)&\in F\subseteq X \quad \forall s\in[0,S] && \text{ (kinematic constraints)} \\
u(s)&\in B(y(s)) \quad \forall s\in[0,S] && \text{ (control constraints)} \\
y^\prime (s)&=D(y(s),u(s)) \,\forall s\in[0,S] && \text{ (dynamic equation)}
\end{alignat}
\end{defn}
This is a highly general formulation.  Note that kinematic planning problems can simply set the control variable $u(s)$ to the derivative of the path, the control set to $B = \{u \quad | \quad \|u\| \leq 1 \}$, and the dynamic equation as $D(y,u)=u$.  Second-order planning problems (i.e., those with inertia) can be defined with a configuration $\times$ velocity state $x=(q,q')$.  Problems with time-variant constraints can be constructed in this same form by augmenting the state variable with the time variable $(x,t)$.

\begin{defn}
\label{def:Optimal}
An {\em optimal} planning problem $P=(L,\Phi,X,U,x_I,G,F,B,D)$ asks to produce a trajectory $y(s):[0,S]\rightarrow X$ that minimizes the objective functional:
\begin{equation}
C(y)=\int^S_0{L(y(s),u(s))ds}+\Phi (y(S))
\end{equation}
among all feasible trajectories (those that satisfy (1--5)).  Here $L$ is the {\em incremental cost} and $\Phi$ is the {\em terminal cost}.
\end{defn}

\begin{defn}
A {\em bounded-suboptimality} planning problem $(P,\epsilon)$ asks to find a trajectory satisfying $C(y)=C^*+\epsilon$, where $C^*$ is the cost of the optimal path and $\epsilon>0$ is a specified parameter.
\end{defn}

\begin{defn}
A {\em complete} planner $A$ finds a feasible solution to a problem $P$ when one exists, and terminates with ``failure'' if one does not.  Moreover, it does so in finite time.  A {\em probabilistically-complete} planner $A$ finds a feasible solution to a problem $P$, when one exists, with probability approaching 1 as more time is spent planning.  A planner $A$ is  {\em asymptotically-optimal} for the optimal planning problem $P$ if the cost $C(t)$ of the generated path approaches the optimum $C^*$ with probability 1 as more time $t$ is spent planning.
\end{defn}

\subsection{State-cost space equivalence}

Our first contribution is to demonstrate an equivalence of any optimal planning problem with that of a canonical \textit{state-cost form} in which the dependence on the incremental cost $L$ is eliminated.  In particular, we augment each state $x$ with the cost $c$ taken to reach it from $x_I$ to derive an expanded state $z=(x,c)$. 

\begin{theorem}
The optimal planning problem $P=(L,\Phi,X,U,x_I,G,F,B,D)$ is equivalent to a state-cost optimal planning problem {\em without incremental costs}
\[\hat{P}=(0,\hat{\Phi},X\times R^+,U,(x_I,0),G\times R^+,F\times R^+,B,\hat{D}),\]
in such a way that solutions to $\hat{P}$ are in one-to-one correspondence with solutions to $P$.  Here the terminal cost $\hat{\Phi}$ is given by
\begin{equation}
\hat{\Phi }\left(\left[ \begin{array}{c}
x \\ 
c \end{array}
\right]\right)=c+\Phi (x)
\end{equation}
And the dynamics $\hat{D}$ are given by
\begin{equation}
z^\prime=\left[ \begin{array}{c}
x^\prime \\ 
c^\prime \end{array}
\right]=\left[ \begin{array}{c}
D\left(x,u\right) \\ 
L\left(x,u\right) \end{array}
\right]
\end{equation} 
\end{theorem}
The proof is straightforward, showing that the projection of solutions to $\hat{P}$ onto the first $dim(X)$ elements are solutions to $P$, and solutions to $P$ can be mapped to solutions to $\hat{P}$ via augmenting them with a cumulative cost dimension.  Note that even if every state in the original space was reachable, not all points in the state-cost space are reachable.  Moreover, the goal set is now a cylinder with infinite extent in the cost direction (Fig.~\ref{fig:StateCostGoals}.a).

Using this state-cost transformation, we derive a corollary that states that a boundedly-suboptimal problem with cost at most $\bar{c}$ can be solved by solving a feasible planning problem (Fig.~\ref{fig:StateCostGoals}.b).  
\begin{corollary}
A bounded-suboptimality planning problem $P$ with $\epsilon = \bar{c}-C^*$
 is equivalent to a feasible planning problem $P_{\bar{c}}=(X\times R^+,U,(x_I,0),G_{\bar{c}},F\times R^+,B,\hat{D})$ where
$G_{\bar{c}}=\{(x,c)\ |\ x\in G,c\in [0,\bar{c}-\Phi (x)]\}$ is the set of terminal state-cost pairs satisfying the goal condition and the cost bound, and $\hat{D}$ is given by the state-cost transformation.
\end{corollary}
Specifically, if $y$ is a solution to $P_{\bar{c}}$, then it corresponds to a feasible solution of $P$ with cost no more than $\bar{c}$ (and no less than $C^*$).  Also, $P_{\bar{c}}$ has no solution if and only if $\bar{c}<C^*$.

\subsection{Bounded-suboptimality meta-planning with a complete feasible planner}

The above corollary suggests that bounded-suboptimality planning is equivalent to feasible kinodynamic planning; however, $C^*$ is a priori unknown.  Hence, we present a bounded-suboptimality meta-planner that repeatedly invokes a feasible planner while lowering an upper bound on cost.  This idea builds some intuition for the asymptotically-optimal planner presented in the following section.

\begin{figure*}
\centering
\begin{tabular}{ccc}
\includegraphics[width=0.3\textwidth]{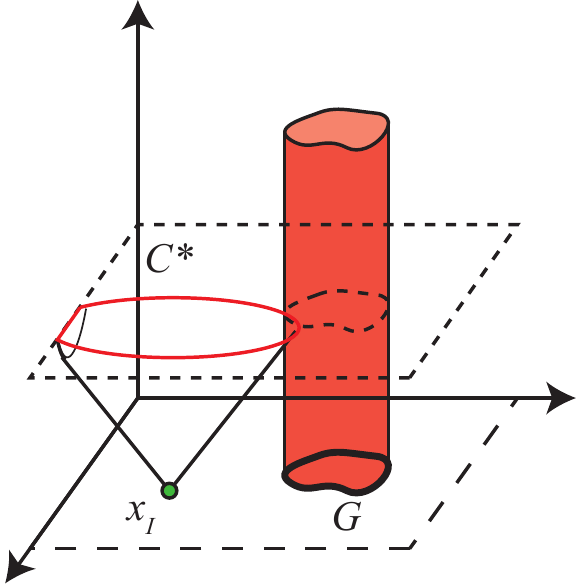} &
\includegraphics[width=0.3\textwidth]{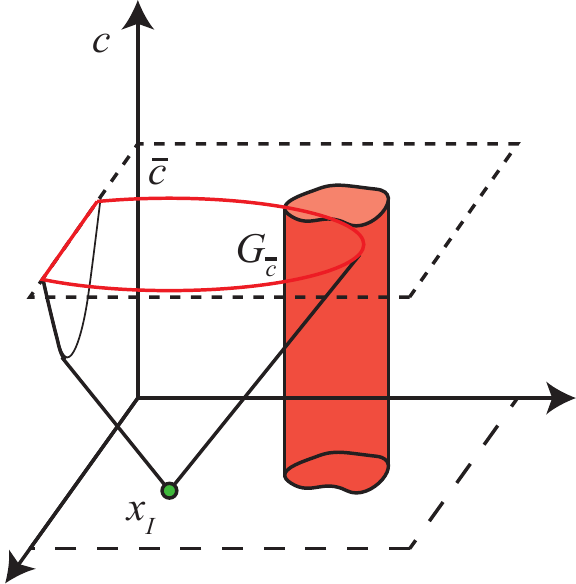} &
\includegraphics[width=0.3\textwidth]{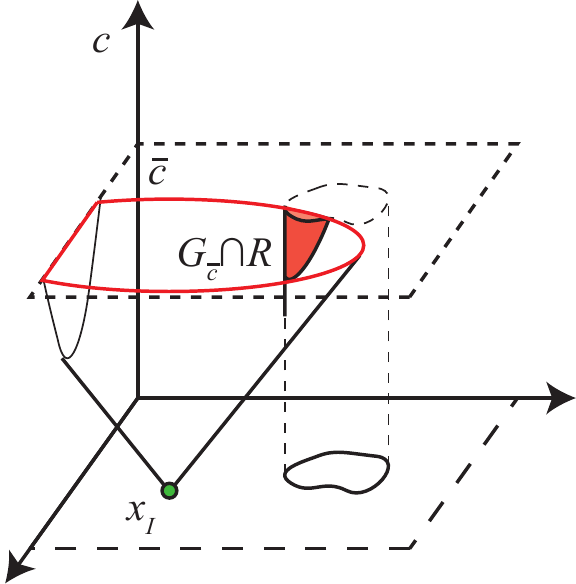} \\
(a) & (b) & (c)
\end{tabular}
\caption{(a) The goal region in state-cost space extends the state-space goal region $G$ infinitely along the cost axis. The optimal cost $C^*$ is the least-cost portion of the reachable set that touches the region.  (b) Given the cost of an existing path $\bar{c}$, the optimal path is known to lie in the region of space with cost lower than $\bar{c}$. Finding a trajectory that improves upon $\bar{c}$ is a feasible planning problem. (c) The reachable portion of this goal set shrinks as $\bar{c}$ approaches the optimum.}
\label{fig:StateCostGoals}
\end{figure*}

The meta algorithm Bounded-Suboptimal($P$,$\epsilon$,$A$) accepts as input a problem $P$, a tolerance $\epsilon$, and a complete feasible planning algorithm $A$, and is listed as follows:

\begin{algorithm}
\caption{Bounded-Suboptimal($P$,$\epsilon$,$A$)}
\label{alg:BoundedSuboptimal}
\begin{algorithmic}[1]
\State Run $A(P_{\infty })$ to obtain a first path $y_0$.  If no solution exists, then report `$P$  has no solution'. 
\State Let $c_0=C(y_0)$.  
\For{$i=1,2,\ldots$ }
  \State Run $A(P_{c_{i-1}-\epsilon})$  to obtain a new solution $y_i$. If no solution to $P_{c_{i-1}-\epsilon}$ exists, then stop.
  \State Let $c_i=C(y_i)$.  
\EndFor
\end{algorithmic}
\end{algorithm}

Step 1 solves for a feasible solution to the original problem, with no limit on cost.  In practice, it may be solved in the original state space simply by discarding the cost function.  In the loop, Step 4 establishes a new cost upper bound by lowering the best cost found so far $c_{i-1}$ by $\epsilon$.  The following theorem proves correctness of this meta-algorithm.

\begin{theorem}
If $A$ is a complete planner for the feasible kinodynamic planning problem, then Bounded-Suboptimal($A$,$\epsilon$) terminates in finite time and produces a path $y_i$ with cost no more than $C^*+\epsilon$.
\end{theorem}
\begin{proof} Let $i+1$ be the index of the final iteration.  In the prior iteration, a solution was found, so $c_{i} \ge C^*$.  In the current iteration, no solution is found, and since $A$ is complete, the current planning problem, $P_{c_{i}-\epsilon}$ is infeasible. So, $c_{i}-\epsilon<C^*$, and therefore $c_{i} = C(y_{i})$ is within $\epsilon$ of optimal.

Running time is finite since each loop reduces the cost by at least $\epsilon$, and hence the inner loop is run no more than $\left\lceil \frac{c_0-C^*}{\epsilon}\right\rceil $ times.
\end{proof}

\subsection{Asymptotically-optimal meta-planning with a randomized feasible planner}

The need for a complete planner is too restrictive for practical use on high-dimensional problems, where only probabilistically complete planners are practical.  Here, we relax this restriction while also eliminating the dependence on the parameter $\epsilon$, under the unrestrictive assumption that the cost is lowered by a nonnegligible fraction whenever $A$ finds a feasible path.

We will need to make some assumptions such that $A$ is ``well-behaved'' so that it has a significant chance of finding a path that shrinks the best cost found so far regardless of $\bar{c}$.  Given some cost upper bound $\overline{c}$, define the cost of the next produced solution follow a cumulative density function $\varphi(C(y);\bar{c})$.  This function ranges from 0 to 1 on the support $[C^*,\bar{c}]$, i.e., $P\left(C(y)\le z\right)=\varphi(z;\bar{c})$.  We {\em do not} prescribe any form for this distribution, however, we do require one condition for its moment.

Well-behavedness of $A$ requires two conditions:
\begin{enumerate}
\item If there exists a feasible solution and $\bar{c} > C^*$, then $A$ terminates in finite time.
\item Given a cost bound $\bar{c}$ expected suboptimality of the computed path is shrunk toward $C^\star$ by a non-negligible amount each iteration.  (In practice, this means that there is a nonzero chance that the planner does not produce the worst-possible path).
\end{enumerate}
Specifically, condition 2 requires that:
\begin{equation}
E[C(y) | \overline{c}]  \leq (1-w) (\overline{c}-C^*)
\label{eq:Shrinkage}
\end{equation}
for some $w > 0$ (the $w$ is required for technical reasons; for most cases this condition enforces that $E[C(y) | \overline{c}]  < (\overline{c}-C^*)$).  This condition is not overly restrictive for most randomized planners; the set of paths with $C(y) = \overline{c}$ is a set with measure zero in the space of paths, and is unlikely to be sampled at random.

We are now ready to present the main algorithm, AO-$x$.

\begin{algorithm}
\caption{Asymptotically-optimal($P$,$A$,$n$)}
\label{alg:AsymptoticallOptimal}
\begin{algorithmic}[1]
\State Run $A(P_{\infty })$ to obtain a first path $y_0$.  If no solution exists, report `$P$  has no solution'. 
\State Let $c_0=C\left(y_0\right)$.  
\For {$i=1,2,\ldots,n$}
  \State  Run $A(P_{c_{i-1}})$  to obtain a new solution $y_i$.
  \State Let $c_i=C\left(y_i\right)$.
\EndFor
\Return $y_n$
\end{algorithmic}
\end{algorithm}

\begin{theorem} If $A$ is a well-behaved randomized algorithm, then Asymptotically-optimal($A,n$) is asymptotically optimal.  In other words, as $n$ approaches infinity, the probability that $y_n$ is not an optimal path approaches zero.
\end{theorem}

\begin{proof}
Let $X_0,\ldots,X_n$ be the nonnegative random variables denoting the suboptimality $C\left(y_i\right)-C^*$ during a run of the algorithm.  We will show that they converge almost surely to the optimum as $n$ increases.  That is we want to show that $P(\mathop{lim}_{n\rightarrow \infty} X_n = 0) = 1.$

Almost sure (a.s.) convergence is equivalent to ${\mathop{\lim }_{n\to \infty } P\left(\mathop{\sup }_{m\ge n} X_m  >\epsilon \right)\ }=0$. Since $\mathop{\sup }_{m\ge n} X_m = X_n$, a.s. convergence is implied by convergence in probability $\mathop{\lim }_{n\to \infty } P(X_n>\epsilon)=0$.  To prove convergence in probability, we will prove $\mathop{\lim }_{n\to \infty } E[X_n] =0$ and then use the Markov inequality $P(X_n\ge \epsilon)\le E[X_n]/\epsilon$.

Conditioning on $X_{n-1}$, we get:
\[E[X_n]=\int E[X_n\ |\ x_{n-1}]P(x_{n-1})dx_{n-1}\] 
and due to \eqref{eq:Shrinkage} we have $E[X_n\ |\ x_{n-1}]\le (1-w) x_{n-1}$.  Hence, 
\begin{equation}
\begin{split}
E[X_n]&=(1-w)\int x_{n-1}P(x_{n-1})dx_{n-1} \\
    &=(1-w)E[X_{n-1}]=(1-w)^n E[X_0]
\end{split}
\end{equation}
and thus,
\begin{equation}
\label{eq:ProbabilityOfSuboptimalityN}
P(X_n\ge \epsilon)\le E[X_0](1-w)^n/\epsilon 
\end{equation}

Clearly this approaches 0 as $n$ increases. 
\end{proof}

\subsection{Convergence rate with respect to time}

We now take a more detailed analysis of the case in which the feasible planner is probabilistically complete, and study the convergence of Asymptotically-Optimal in terms of running time $t$ rather than the number $n$ of planner calls. We show again, under relatively weak assumptions, that Asymptotically-Optimal is asymptotically optimal in terms of time, even though each call to the planner takes increasingly longer to complete as $n$ increases because the reachable portion of the goal set shrinks (Fig.~\ref{fig:StateCostGoals}.c).

A planner is probabilistically complete if the probability that it finds a feasible path, if one exists, approaches 1 as more time is spent planning.  Note that a probabilistically complete planner will not necessarily terminate if no feasible path exists.

Note that probabilistic completeness is not a sufficient condition for a planner to be useful, since the convergence rate may be so slow that it is impractical.  As an example, let $A$ be a probabilistically complete planner, and $f(t)$ denote $P(\text{A fails given $t$ seconds of planning})$.  If $f(t) = 1/t$, then expected running time is \textit{infinite}. 

We will assume that for the given $X$, $x_I$, $F$, and $\hat{D}$ the planner $A$ satisfies an exponential convergence bound, in which $f(t)\le \max \left(1,\alpha e^{-\beta t}\right)$ for some positive values  $\alpha$ and $\beta$. In practice, an exponential convergence bound implies expected running time is finite.
\[E[t]\le \int^{\infty }_0{t f(t) dt }\le \frac{a}{\beta}\] 
A more refined analysis~\cite{HL10} gives a tighter bound
\begin{equation}
\begin{split}
\int^{\infty }_0{tf(t)dt } = & \int^{(\ln \alpha)/\beta}_0{ t dt }+\int^{\infty }_{(\ln  a)/\beta}{t\alpha e^{-\beta t}dt} \\
 = & \frac{\ln  \alpha}{\beta}+\frac{1}{1-e^{-\beta}}
\end{split}
\end{equation}

Convergence rate varies, however, depending on the {\em reachable} portion of the goal region $G_{\overline{c}} \cap R(x_I)$ where $R(x)$ is the reachability set of $x$ in state-cost-space (Fig.~\ref{fig:StateCostGoals}.c).  In particular, a small goal region makes it rare to $A$ to sample a configuration in it at random, which slows convergence.  Hence, the convergence rates are properly defined as a function of the volume of the goal region:

\begin{equation}
\alpha \equiv \alpha(\mu(G_{\overline{c}}\cap R(x_I))), \beta\equiv \beta (\mu (G_{\overline{c}}\cap R(x_I)))
\end{equation}
The following theorem gives an example of such a bound when the EST algorithm is used as the underlying feasible planner.

\textbf{Theorem}. Assuming the space is \textit{expansive}, the Kinodynamic EST planner (Hsu, Latombe, Kindel, and Rock 2001) satisfies an exponential convergence bound with constants $\alpha(g)=\gamma \ln  \frac{1}{g}$ and $\beta(g)=\delta g$ for positive constants $\gamma$ and $\delta$, where $g$ is the volume of the reachable goal region.  Moreover, $E[t]$ is $O\left(\frac{1}{g}{\ln  {\ln  \frac{1}{g}} }\right)$ as $g$ approaches 0.

\begin{proof}  From Hsu, Latombe, Kindel, and Rock 2002~\cite{HKLR02}, Kinodynamic EST with a uniform sampling strategy fails to find a path with probability no more than $p$ if at least $\frac{k}{\alpha} \ln  \frac{2k}{p} +\frac{2}{g} \ln \frac{2}{p}$ milestones are sampled, where $k=\frac{1}{\beta} \ln \frac{2}{g}$ and $\alpha$ and $\beta$ are expansiveness constants that are fixed for the given configuration space (not related to the $\alpha$ and $\beta$ defined above). Using a bit of algebra, this expression can be rewritten as a bounded probability of failure:
\[f(t)\le {\left(2k\right)}^{\frac{kg}{kg+2\alpha}}\cdot 2^{\frac{2\alpha}{kg+2\alpha}}\cdot e^{\frac{-t \alpha g}{kg+2\alpha}}\] 
Here we have assumed that each sample takes constant time and the constant factor is ignored. 

Now we will simplify this rather unwieldy expression.  First, note that the exponents of the first two terms in the equation are upper bounded by 1 since they are ratios of two positive numbers to their sums, and hence
\[f(t)\le 4k\cdot e^{\frac{-t \alpha g}{kg+2\alpha}}\] 
Next, we use the fact that $\ln x \le x$.  Hence, $k=\frac{1}{\beta} \ln  \frac{2}{g} \le \frac{2}{\beta g}$.  The factor $\frac{\alpha g}{kg+2\alpha}$ in the exponent can now be lower bounded by $\frac{\alpha\beta}{2+2\alpha\beta}g$ and hence we have the desired expression
\begin{equation}
f(t)\le \gamma \ln  \frac{1}{g} e^{-t\delta g}\
\end{equation}
With $\gamma=\frac{8}{\beta}$ and $\delta=\frac{\alpha\beta}{2+2\alpha\beta}$ constant for a given configuration space. As a result, the running time is bounded by $E\left[t\right]\le \frac{1}{\delta g}\left[{\ln \gamma}+{\ln  {\ln  \frac{1}{g} }}\right]+\frac{1}{1-e^{-\delta g}}$. As $g$ shrinks, the latter term's order of convergence is $\frac{1}{g}$, so we can conclude that $E[t]$ is $O\left(\frac{1}{g}{\ln  {\ln  \frac{1}{g} } }\right)$ as desired.
\end{proof}

\begin{figure}
\centering
\includegraphics[width=0.9\linewidth]{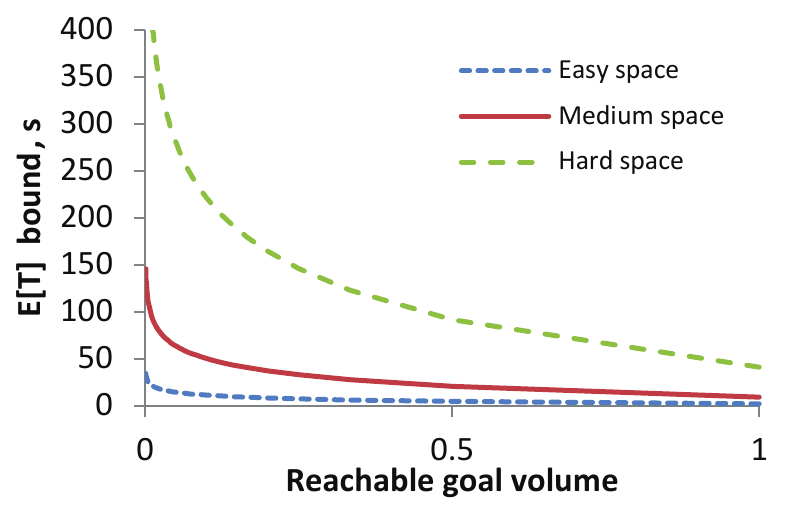}
\caption{Running time for each invocation of the underlying planning subroutine increases asymptotically to infinity as the reachable goal volume decreases.  The visibility characteristics of the underlying space also have major effects on running time. Plots illustrate the theoretical bounds on expected running time of EST for spaces of three different ``difficulty'' levels.  Specifically, visibility constants $\alpha$ and $\beta$ are set to 0.04, 0.02, and 0.01 for easy, medium, and hard spaces.  It is assumed that 1,000 samples are generated per second.}
\label{fig:ESTTimeGoalVolume}
\end{figure}

Fig.~\ref{fig:ESTTimeGoalVolume} illustrates this bound.  It is apparent that, since $G$ is fixed by the original problem, $G_{\overline{c}}$ varies only with the parameter $\overline{c}$.  Hence we may also state these functions as $\alpha(\overline{c})$ and $\beta(\overline{c})$.  In order for EST to be ``well-behaved'' as defined above, we must require that as $\overline{c}$ approaches $C^*$, the volume of $G_{\overline{c}} \cap R(x_I)$ is nonzero as long as $\overline{c} > C^*$.

Let us now state our main result regarding AO-$A$, which is the planner defined as Asymptotically-Optimal($P,A,\infty $).

\begin{theorem}If $A$ is a probabilistically complete, exponentially convergent planner, then AO-$A$ is asymptotically optimal in total running time $t$.
\end{theorem}

\begin{proof}
Define $c(t)$ as the cost of the best path found so far in AO-$A$ after time t has elapsed.  We will show that it converges almost surely toward C* as t increases.  Let Z$(t)$ be the random variable denoting the suboptimality $c(t)-C^*$ during a run of Asymptotically-optimal($A,\infty $). We wish to prove that $\mathop{\lim }_{t\to \infty } P(Z(t)\ge \epsilon)=0$  for any $\epsilon>0$. 

Let $T(x)$ be the random variable denoting the time at which the cost of the best path found so far decreases below $x+C^*$. It is evident that  $P\left(Z(t)\ge \epsilon\right)=P\left(T(\epsilon)\ge t\right)$.  We wish to show that $E\left[T(\epsilon)\right]<\infty $, which would in turn imply the claim due to the Markov inequality $P\left(T(\epsilon)\ge t\right)\le E\left[T(\epsilon)\right]/t$ .

If the current cost bound is greater than $\epsilon$, then the expected time to find a path for any iteration is upper-bounded by the expected time it would take to find a path to $G_{C^*+\epsilon}$, which is some finite value $t_\epsilon$ since $A$ is exponentially convergent. Specifically, $t_\epsilon=\frac{\epsilon}{\delta\epsilon^2}$ for kinodynamic EST as shown above.  So, if Asymptotically-optimal first finds a path with suboptimality no more than $\epsilon$ on the i'th iteration, then the cost expended is no more than $it_\epsilon$.

If we let $N$ denote the random variable of the iteration on which Asymptotically-optimal first finds a path with suboptimality no more than $\epsilon$, then we can bound $E\left[T(\epsilon)\right]$ as follows:
\[E\left[T(\epsilon)\right]\le \sum^{\infty }_{i=0}{it_\epsilon P\left(I=i\right)}=t_\epsilon \sum^{\infty }_{i=0}{iP\left(N=i\right)}=t_\epsilon E[N]\] 
To show that $E[N]$ is finite, we will take a variant of the proof of Theorem 3.  Again let $X_0,\dots ,X_n$ be the nonnegative random variables denoting the suboptimality $C\left(y_i\right)-C^*$ during a run of the algorithm.  $N$ is the index of the first $X_i$ that decreases below $\epsilon$. Hence,
\[P(N\le i)=P(X_i\le \epsilon)\] 
and
\begin{equation}
\begin{split}
P(N=i)& =P(N\le i)-P(N\le i-1) \\
 & =P(X_i\le \epsilon)-P(X_{i-1}\le \epsilon)
\end{split}
\end{equation}

We will use the cumulative probability function definition:
\[P(X_i\le z|X_{i-1})=\varphi(C^*+z;C^*+X_{i-1})\] 
and begin by conditioning on $X_{i-1}$.
\begin{equation}
\begin{split}
P&\left(X_i\le \epsilon\right)=  \int^{\infty }_0{P\left(X_i\le \epsilon | X_{i-1}\right)dP\left(X_{i-1}\right)} \\
 = &
\int^\epsilon_0{P\left(X_i\le \epsilon\ |\ X_{i-1}\right)dP\left(X_{i-1}\right)}  \\
     & +\int^{\infty }_\epsilon{P\left(X_i\le \epsilon | X_{i-1}\right)dP(X_{i-1})} \\
 = &
\int^\epsilon_0{1\cdot dP\left(X_{i-1}\right)}+\int^{\infty }_\epsilon{\varphi(C^*+\epsilon;C^*+X_{i-1})dP(X_{i-1})} \\
= & P\left(X_{i-1}\le \epsilon\right)+\int^{\infty }_\epsilon{\varphi(C^*+\epsilon; C^*+X_{i-1})dP(X_{i-1})}
\end{split}
\end{equation}
Hence, 
\begin{multline}
P(N=i)=\int^{\infty }_\epsilon{\varphi(C^*+\epsilon; C^*+X_{i-1} )dP\left(X_{i-1}\right)} \\
\le \int^{\infty }_\epsilon{dP\left(X_{i-1}\right)}=P\left(X_{i-1}\ge \epsilon\right)
\end{multline}

Where we have applied the bound $\varphi(x; y)\le 1$ which holds because $\varphi$ is a CDF.

We can now apply equation \eqref{eq:ProbabilityOfSuboptimalityN} derived in Theorem 3 to obtain $P(N=i)\le E\left[X_0\right](1-w)^{i-1}/\epsilon$ where $(1-w)<1$ is defined as before.  This upper bound has the form of a geometric distribution with parameter $w$.  So, $E[N]\le \frac{(1-w)}{w}E[X_0]/\epsilon$ which is finite.
\end{proof}

To be specific, if we were to use the exponential convergence bound for kinodynamic EST, we may conclude that $E\left[T(\epsilon)\right]$ is $O\left(\frac{1}{\epsilon^2}{\ln {\ln \frac{1}{\epsilon}}}\right)$.  

We note that this convergence bound is rather loose; earlier iterations will likely terminate much faster than $t_\epsilon$.

\subsection{Complexity Discussion}

The computational complexity of AO-$x$ is affected by several aspects of problem structure.  As remarked before, the visibility characteristics of the problem affect the running time of the feasible planning subroutine $x$.  As a result, the optimal parameters of $x$, such as the expansion distance in RRT, are problem dependent.  

We also note that planner performance in state space may be different from performance in (state, cost) space.  Adding a dimension of cost may increase both time and space complexity, and it also adds drift to problems that may originally be driftless.  We note, however, that the control space remains unchanged, and the performance of many planners are governed chiefly by control complexity.

Lastly, we observe that problem dimensionality does not have a direct relationship to the order of convergence of AO-$x$.  However, it does have a large impact in the running time of $x$, which is manifest in the terms $t_\epsilon$ and $w$ in the proof above.  Problems of higher dimension will tend to have a larger value of the term $t_\epsilon$, although it is easy to construct hard low dimensional problems.  The expected cost reduction $w$ is also dimensionality-dependent; for example, if the reachable goal region in (state, cost) space is locally shaped at the optimum like a convex cone of dimension $d$, then a goal configuration sampled at random will achieve an average cost reduction of $O(1/(d+1))$.

\section{Implementations and Experiments}

This section describes the application of AO-$x$ to several example problems using the feasible kinodynamic planners EST and RRT.  We will refer to the implementations as AO-EST and AO-RRT.  All planners are implemented in the Python programming language, and hence could be sped up greatly by the use of a compiled language.

\subsection{Implementations using RRT and EST}

Both kinodynamic EST and kinodynamic RRT are tree-growing planners that perform random extensions to a state-space tree, rooted at the start, by sampling a node in the tree and a control at random, and then integrating the dynamics forward over a short time horizon.  They differ by sampling strategy.  EST attempts to sample an extension so that its terminal state is uniformly distributed over the reachable set of the current tree.  RRT attempts to sample an extension so that it is pulled toward a randomly-sampled state in state space (a Voronoi bias).  Both methods can also incorporate goal biasing strategies to avoid excessive exploration of the state space in directions that are not conducive to reaching the goal.

{\bf EST Implementation.} EST can be applied directly to state-cost planning.  To approximate sampling over a uniform distribution over the tree's reachable set, it samples extensions with probability proportional to the inverse density of existing states in the tree.  We use the standard method to approximate density by defining a grid of resolution $h$ and low dimension $k$ over randomly chosen orthogonal projections of the state-cost space.  The density of a state $x$ is estimated as proportional to the number of nodes in the tree $N(x)$ contained in the same grid cell as $x$.  In a manner similar to locality sensitive hashing, we choose several grids and count the total number of nodes sharing the same cell as $x$ across all grids.  For our experiments, we use ${dim(X)+1}\choose{k}$ grids, $h=0.1$, and $k=3$, and scale the configuration space $X$ to the range $[0,1]^{dim(X)+1}$ before performing the random projection.  To extend the tree we sample 10 candidate extensions by choosing 10 source states uniformly from the set of occupied grid cells, and drawing one random control sample.  Among those extensions that are feasible, we select one with probability proportional to $1/(N(x_t)+1)^2$ where $x_t$ is its terminal state.  

{\bf RRT Implementation.} RRT can also be applied almost directly, but there are some issues to be resolved regarding the definition of a suitable distance metric. RRT relies on a distance metric to guide the exploration toward previously unexplored regions of state space, and is rather sensitive to the choice of this metric, with better performance as the metric approximates the true cost-to-go.  However, cost-to-go is usually difficult to estimate accurately particularly in the presence of complex obstacles and dynamic constraints.  Below, we empirically investigate the effects of the distance metric.  Nearest node selection is accelerated using a KD-tree data structure.

{\bf Performance considerations.} In both cases, rather than planning from scratch each iteration, we maintain trees from iteration to iteration, which leads to some time savings.  We also save time by pruning the portion of the tree with cost more than $\bar{c}$ whenever a new path to the goal is found.  Specifically, smaller trees make EST density updates and RRT nearest neighbor queries computationally cheaper, although RRT benefits more from this optimization because a larger fraction of its running time is spent in nearest neighbor queries.  We also prune more aggressively if a heuristic function $h(x)$ is available.  If $h(x)$ underestimates the cost-to-go, then we can prune all nodes such that $c+h(x) > \bar{c}$.  Other sampling heuristics could also be employed to bias the search toward low-cost paths~\cite{akgun2011sampling}.

\subsection{Example Problems}

\begin{figure}
\centering
\fbox{\includegraphics[width=0.45\linewidth]{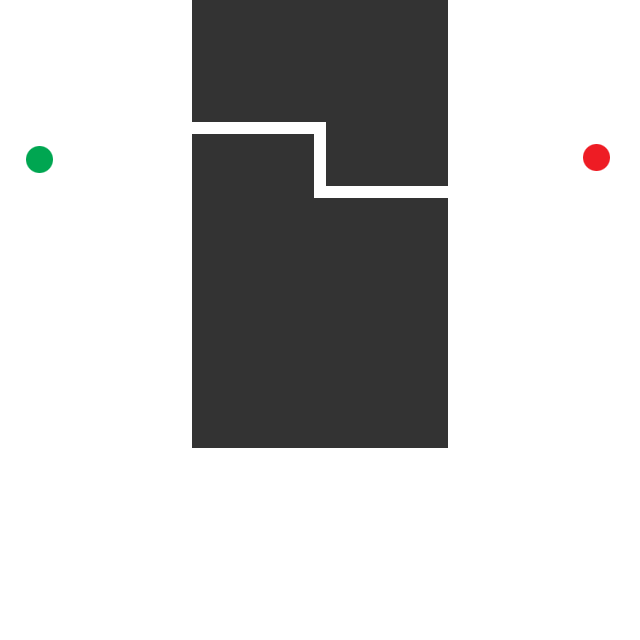} }
\fbox{\includegraphics[width=0.45\linewidth]{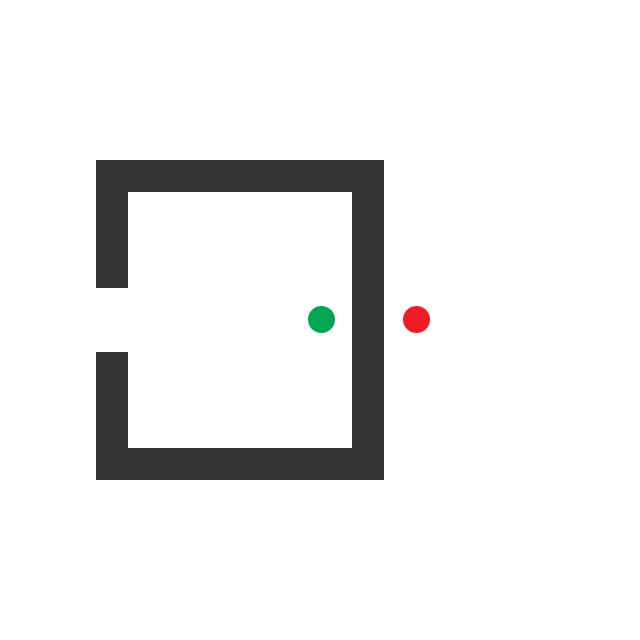} }
\caption{The Kink and Bugtrap problems ask to find the shortest path between the two indicated configurations.}
\label{fig:PlanarProblems}
\end{figure}

{\bf Kink.} The Kink problem (Fig.~\ref{fig:PlanarProblems}.a) is a kinematically-constrained problem in a unit square $[0,1]^2$ in which the optimal solution must pass through a narrow corridor of width 0.02 with two kinks.  The objective is to minimize path length.  Most planners very easily find a suboptimal homotopy class, but it takes longer to discover the optimal one.  The maximum length of each expansion of the tree is limited to 0.15 units. 

{\bf Bugtrap.} The Bugtrap problem (Fig.~\ref{fig:PlanarProblems}.b) is a kinematically-constrained problem in a unit square $[0,1]^2$ that asks the robot to escape a local minimum.  The objective function is path length. This is a challenging problem for RRT planners due to their reliance on the distance metric as a proxy of cost.  The maximum length of each expansion of the tree is limited to 0.15 units.

\begin{figure*}
\centering
\fbox{\includegraphics[width=0.23\textwidth]{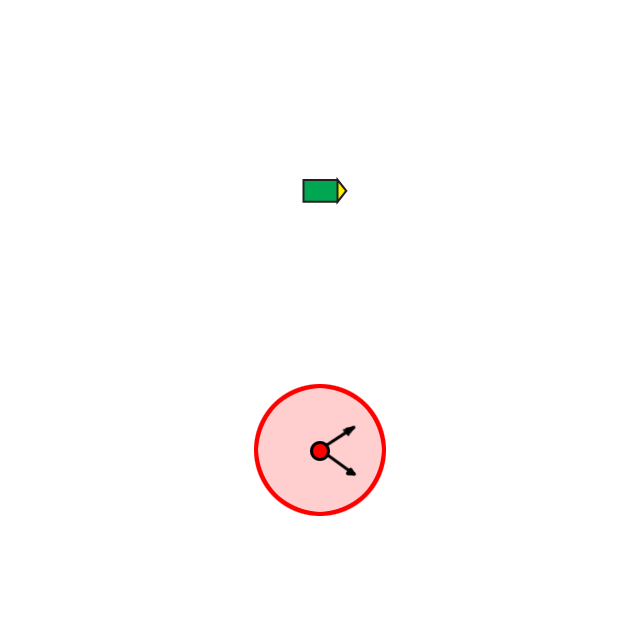}}
\fbox{\includegraphics[width=0.23\textwidth]{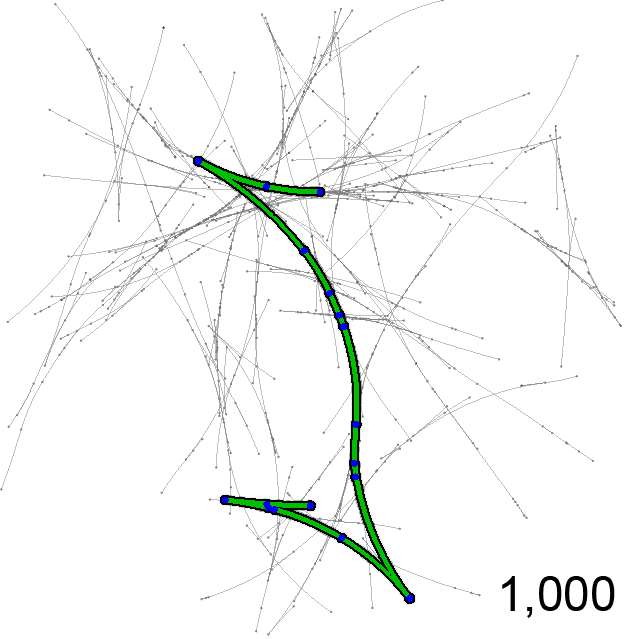}}
\fbox{\includegraphics[width=0.23\textwidth]{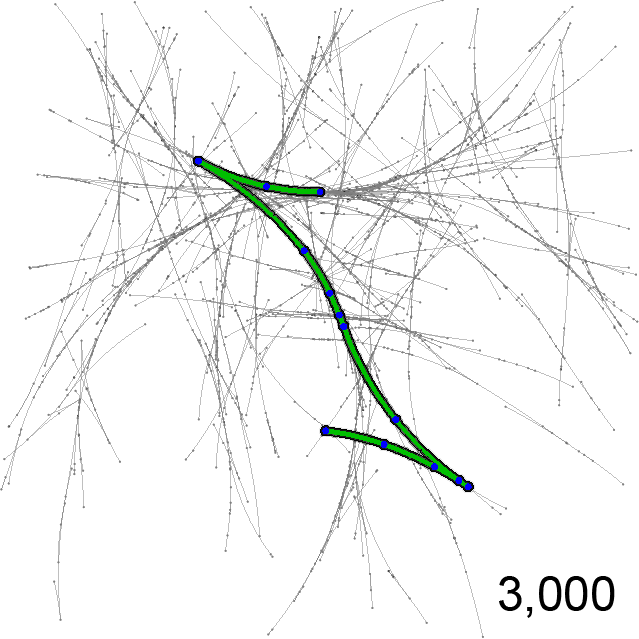}}
\fbox{\includegraphics[width=0.23\textwidth]{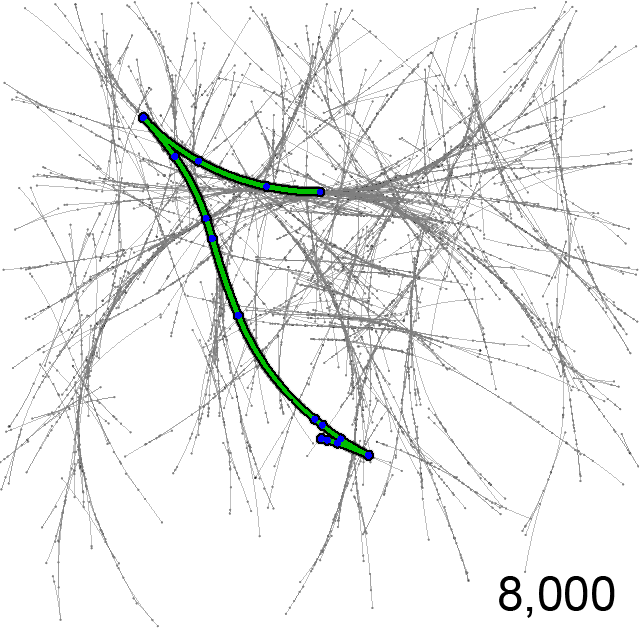}}\\
\vspace{1mm}
\fbox{\includegraphics[width=0.23\textwidth]{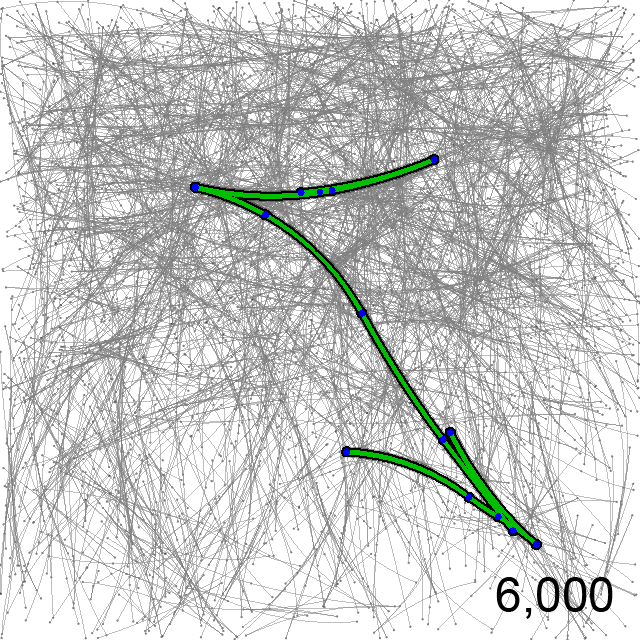}}
\fbox{\includegraphics[width=0.23\textwidth]{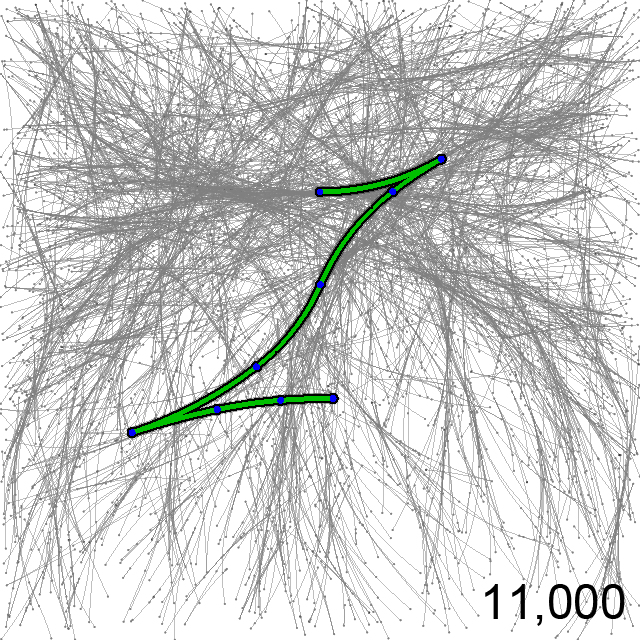}}
\fbox{\includegraphics[width=0.23\textwidth]{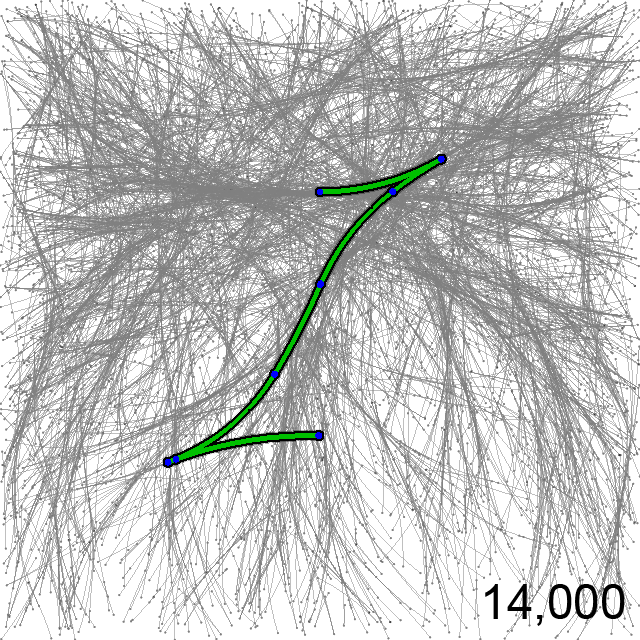}}
\fbox{\includegraphics[width=0.23\textwidth]{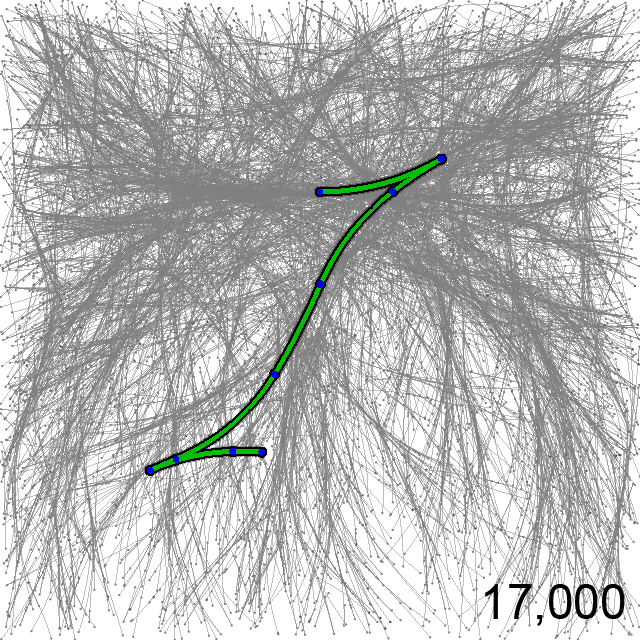}}%
\caption{Planning a sideways maneuver for a Dubins car using AO-RRT (top row) and AO-EST (bottom row). Numbers indicate total number of planning iterations.  Green curve indicates best path found so far. Iteration counts are not directly comparable because RRT spends more time per iteration. }
\label{fig:Dubins}
\end{figure*}

{\bf Dubins.} This problem asks to move a standard Dubins car sideways while keeping orientation relatively fixed (Fig.~\ref{fig:Dubins}). The state is $(x,y,\theta)$ and the control is $(v,\phi)$ where $\theta$ is the heading, $v$ is the forward velocity, and $\phi$ is the steering angle.  State constraints include $(x,y) \in [0,1]^2$, $v \in \{-1,+1\}$, and $\phi \in [-\pi,\pi]$.  For planning, time steps are drawn at random from $[0,0.25]$\,s.  The metric is $d((x,y,\theta),(x^\prime,y^\prime,\theta^\prime)) = \sqrt{(x-x^\prime)^2 + (y-y^\prime)^2 + d_\theta(\theta,\theta^\prime)^2 / (2\pi) }$ where $d_\theta$ measures the absolute angular difference.  The goal is to move the car sideways 0.4 units with a tolerance of 0.1 units in state space, with minimal execution time (equivalent to minimum path length).

{\bf Double Integrator.} This asks to move a point with bounded velocities and accelerations to a target location.  The state space includes $x=(q,v)$ includes configuration $q$ and velocity $v$, with constraints $q \in [0,1]^2$, $v \in [-1,-1]^2$, and $u \in [-5,5]^2$, with $\dot{q} = v$ and $\dot{v} = u$.  The start is 0.06 units from the left and the goal is 0.06 units from the right, which must be reached with a tolerance of 0.2 units in state space.  Distance is euclidean distance.  Time steps are drawn from $[0,0.05]$\, s.

{\bf Pendulum.} The pendulum swing-up problem places a point mass of $m=$1\,kg at the end of a $L=$1\,m massless rod. The state space is $x=(\theta,\omega)$.  The goal is the set of states such that the rod is within $10^\circ$ of inverted and absolute angular velocity less than 0.5\,rad/s, and the cost is the total time required to complete the task.  We take gravitational acceleration to be $g=$9.8\,N$\cdot$s$^2$, and a motor can exert a torque at the fixed end of the rod with bang-bang magnitudes $\tau\in\{-2,0,2\}$\,N$\cdot$m. The dynamics of the system are described by:
\begin{align}
\dot{\theta} &= \omega \\
\dot{\omega} &= \frac{\tau-mg\sin(\theta) L}{mL^2} = -9.8\sin(\theta)+\tau
\end{align}
The difficulty in this task arises from the fact that the exerted torque cannot make the pendulum complete a full rotation. In fact, the torque will be canceled by gravity at about $11.5^\circ$. Therefore, the only way to achieve an inverted position is to take the advantage of gravity by swinging back and forth and accumulating angular momentum. 
For planning, constant torques are applied for a uniformly chosen duration between 0 and 0.5\,s, and trajectories are numerically integrated using a time step of 0.01\,s.  Figure \ref{fig:pendulum-angle-plot} shows the first 5 paths obtained by AO-RRT. 
\begin{figure}
\centering
\includegraphics[width=0.8\linewidth]{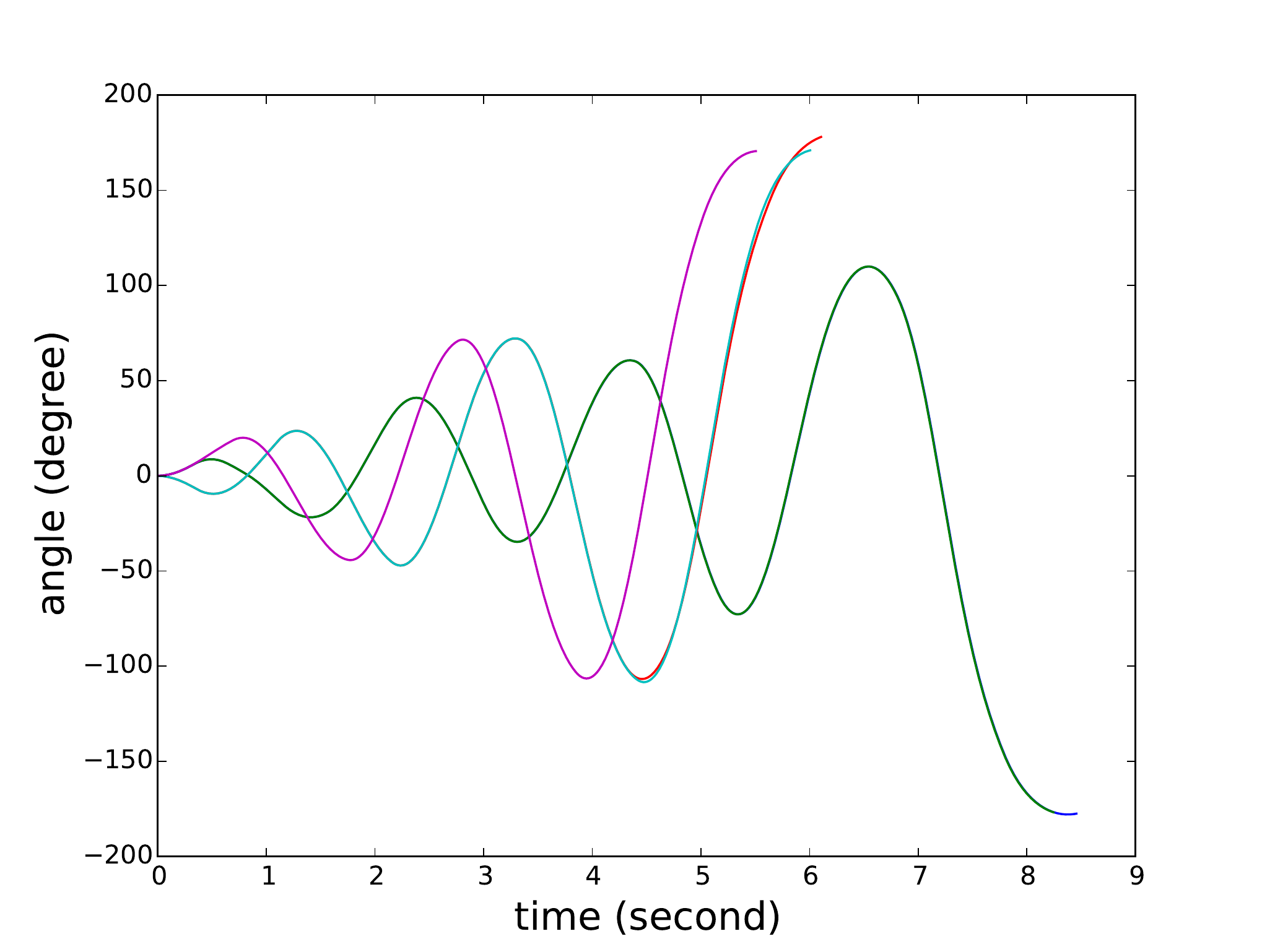}
\caption{Plot of angle vs. time for the first five trajectories found by AO-RRT on the pendulum example. Shorter execution times (rightmost point on each curve) are preferred.  The execution time decreases from 8.46 seconds in the first solution to 5.51 seconds in the fifth solution. (Best viewed in color)}
\label{fig:pendulum-angle-plot}
\end{figure}

{\bf Flappy.} We devised a simplified version of the once-popular game Flappy Bird.  The ``bird'' has a constant horizontal velocity, and can choose to fall freely under gravity, or apply a sharp upward thrust. The trajectory is a piecewise-parabolic curve.  In the original game the objective is simply to avoid obstacles as long as possible, but in our case we consider other cost functions.  The goal is to traverse from the left of the screen to a goal region on the right. The screen domain is $1000\times 600$ pixels with fixed horizontal velocity of $v_x=5px/s$. The gravitational acceleration is $g=1px/s^2$ downward. The control $u$ is binary, and provides an upward thrust of either $0$ or $4px/s^2$. Fig.~\ref{flappy-demo} shows an example solution path obtained by our planner.
\begin{figure}
\begin{center}
\includegraphics[width=0.85\linewidth]{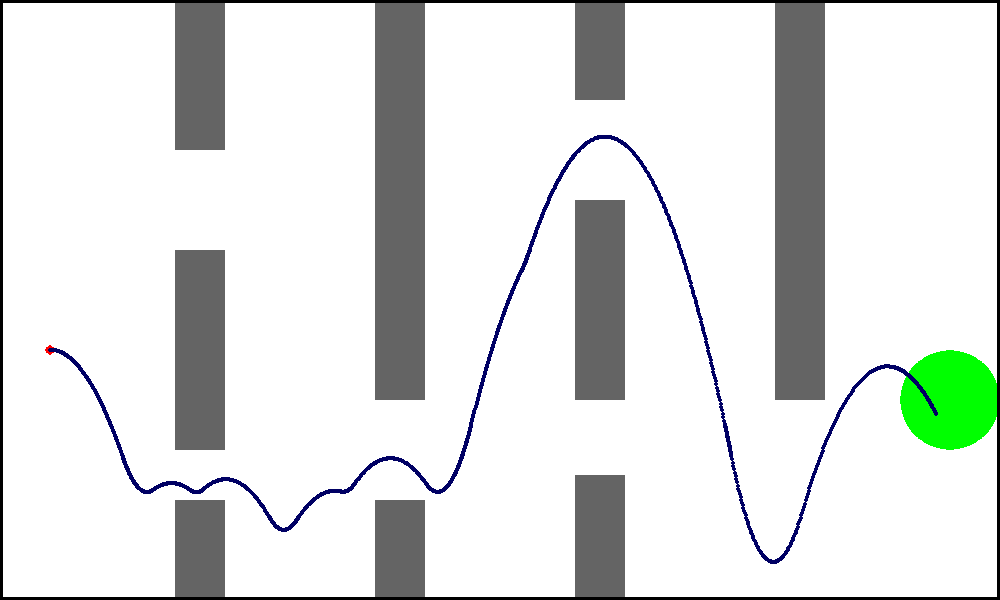}
\caption{Example solution path for Flappy. The planner finds a path that goes from the starting point on the left to the green goal region on the right while avoiding obstacles represented by gray rectangles. }
\label{flappy-demo}
\end{center}
\end{figure}
We represent the state by
\begin{align}
\vec{x}=\begin{bmatrix}
x\\
y\\
v_y
\end{bmatrix}
\end{align}
where $(x,y)$ is the bird position and $v_y$ is vertical velocity. The state evolves according to $\dot{x}=5$, $\dot{y}=v_y$, and $\dot{v_y} = -1+4u$ where $u\in\{0,1\}$ is the binary control.  Time steps are sampled uniformly from the range $[0,1]$, and the time evolution of the state is solved for analytically.  The experiments below  illustrate the ability of AO-$x$ to accept unusual cost functions.

\subsection{Experiments}

{\bf Comparing AO-EST and AO-RRT.} Fig.~\ref{fig:Dubins} illustrates AO-EST and AO-RRT applied to the Car example.  Qualitatively, RRTs tend to explore more widely at the beginning of planning, while ESTs tend to focus more densely on regions already explored.  As a result, in this example, AO-RRT finds a first path quicker, while AO-EST converges more quickly to the optimum (each iteration of EST is cheaper).  Like in feasible planning, the best planner is largely problem-dependent, and we could find no clear winner on our other experiments.  

\begin{figure*}
\centering
\includegraphics[width=0.98\textwidth]{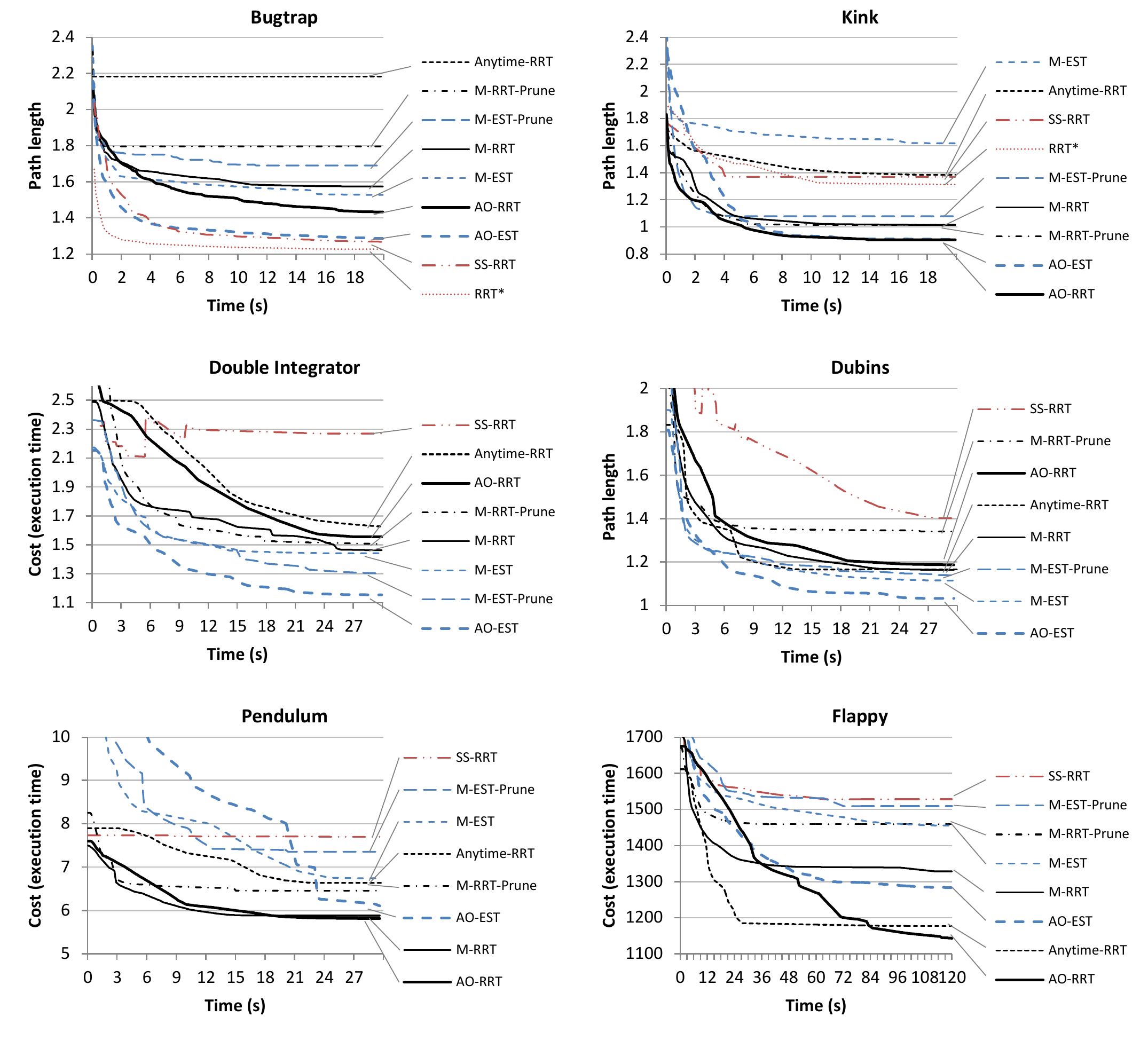}
\caption{Results of benchmark tests.  Curves measure solution cost vs computation time, averaged over 10 runs. (Lower is better)}
\end{figure*}

{\bf Benchmarking against comparable planners.}  We compare against the simpler meta-planner M-$x$ which simply runs the feasible planner $x$ multiple times, keeping the lowest-cost path found so far.  We also experimented with a variant, M-$x$-Prune, which prunes search nodes whose cost is greater than the cost of the best path found so far.  In the 2D problems, we compare against RRT*~\cite{KF2010}, and for fair comparison we provide the other RRT-based planners with the straight-line a steering function as well.  We also compare against Anytime-RRT~\cite{Ferguson2006anytime} and Stable-Sparse-RRT (SS-RRT)~\cite{LLB2014}.  We also compared SST*~\cite{LLB2014}, but it performed worse than SS-RRT in all of our tests. 

For fair comparison, all algorithms were implemented in Python using the same subroutines for feasibility checking, visibility checking, and distance metrics.  All planners used the same parameters as AO-$x$ where applicable.  For Anytime-RRT we used $\epsilon=0.01$ and $\delta_c = 0.1$, and found performance was relatively insensitive to these parameters.  For SS-RRT, we used parameters $\delta_{BN} = 0.1$ and $\delta_s = 0.03$.  Tuning of these parameters did not seem to have a consistent effect on performance.  KD-trees were used for closest node selection in all of the RRT-based algorithms except Anytime-RRT, in which brute-force selection must be used because it does not select nodes using a true distance metric. 

Fig.~\ref{Benchmarks} displays computation time vs solution cost, averaged over 10 runs for all of the benchmark problems.  These results suggest that AO-EST consistently outperforms M-EST and M-EST-Prune, while AO-RRT sometimes outperforms M-RRT and M-RRT-Prune, but sometimes performs roughly the same.  We find that Anytime-RRT and SS-RRT typically do not perform even as well as the simpler $M-RRT$ algorithm, although Anytime-RRT did perform well on Flappy, and SS-RRT did perform well on Bugtrap.  Surprisingly, RRT*  performed excellently on Bugtrap but poorly on Kink despite the fact that it uses rewiring via a steering function. This drop in performance is explained by the fact that it spends excessive amounts of time building a detailed roadmap of the open homotopy class, rather than exploring the narrow passage.

Overall, we observe that AO-EST is best or near-best performer in most problems.  AO-RRT sometimes is the best performer, but is more inconsistent.  A possible explanation is the well-known {\em metric sensitivity} of RRTs: when the distance metric becomes a poor approximation to cost-to-go, then RRT performance deteriorates.  This property is inherited by AO-RRT.

{\bf RRT distance metric.}  We empirically studied the influence of distance metric on planning time and quality for AO-RRT.  
For the pendulum example, we use a weighted Euclidean metric
\begin{align}
d(x_1,x_2)=\sqrt{d_\theta(\theta_1,\theta_2)^2+(\omega_1-\omega_2)^2+w_c(c_1-c_2)^2}
\end{align}
where $w_c$ trades off between the state-space distance and the cost-space distance. For each value of $w_c =$ 0.1, 0.3, 1, 3, and 10, we ran AO-RRT 10 times using a 60\,s time limit.  Fig.~\ref{fig:cost-weight-pendulum} shows that for this example, higher cost weights have a minor effect on solution cost but a detrimental effect on running time per iteration.

\begin{figure}
\centering
\includegraphics[width=0.47\linewidth]{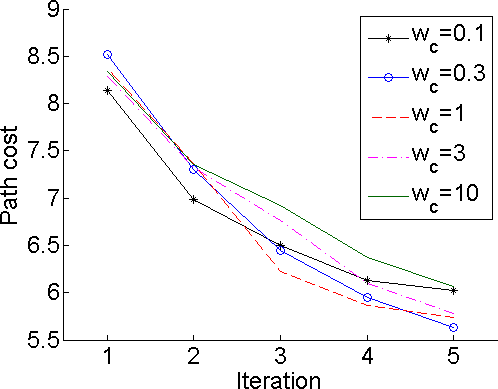}
\includegraphics[width=0.47\linewidth]{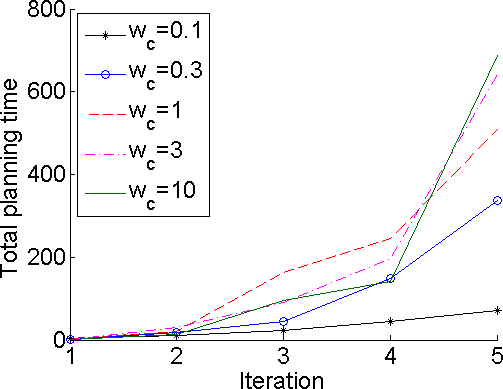}
\caption{The influence of cost weight in the RRT distance metric for the Pendulum example showing average cost (left) and planning time (right) over successive iterations.}
\label{fig:cost-weight-pendulum}
\end{figure}

\begin{figure}
\centering
\includegraphics[width=0.7\linewidth]{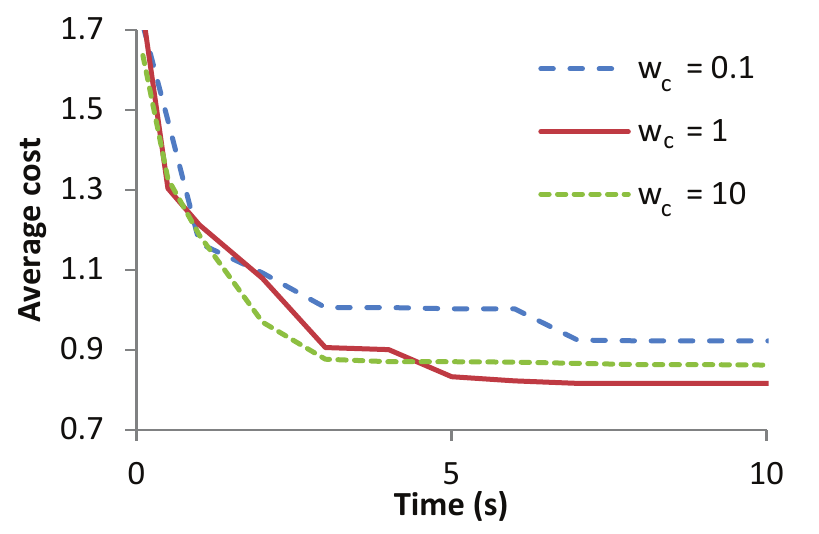}
\caption{Experiments comparing the cost weight on a planar kinematic problem suggests that estimating the true cost leads to faster convergence.}
\label{fig:cost-weight-kinematic}
\end{figure}

We found a very different effect on a second problem.  This one is a kinematically-constrained, planar minimum path length problem with obstacles.  The ``ideal'' cost weight is 1, since it perfectly measures the cost-to-go.  Experiments in Fig.~\ref{fig:cost-weight-kinematic} justify this choice, showing that it converges quicker toward the optimum.

{\bf Adaptation to different costs.} Using the Flappy problem, we demonstrate the fast adaptability of the AO method to different cost functions, even those that are non-differentiable.  AO-RRT is used here. First, we set cost equal to path length.  The second cost metric penalizes the distance traveled {\em only in the lower half} of the screen. The optimal path prefers high altitudes and passes through the two upper openings and one lower opening.  Fig.~\ref{fig:flappy-costs}, shows the results.
\begin{figure}
\centering
\includegraphics[width=0.85\linewidth]{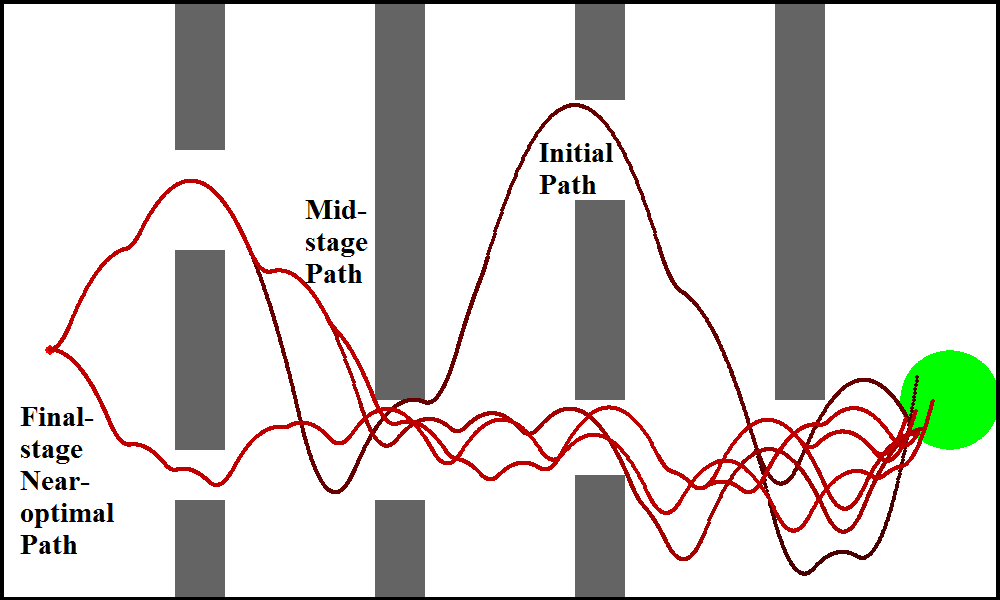}\\
\vspace{1mm}
\includegraphics[width=0.85\linewidth]{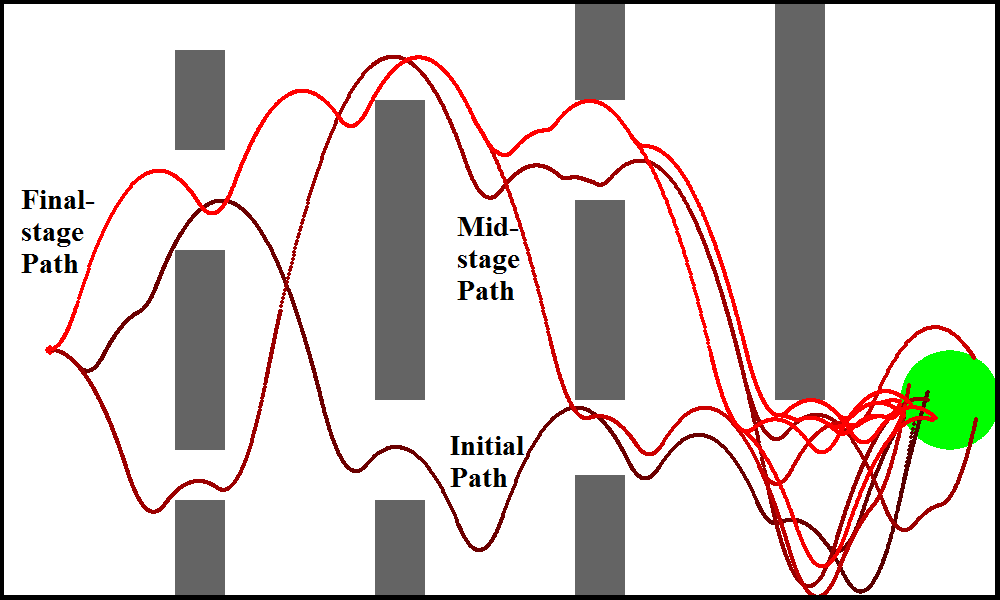}%
\caption{Convergence of Flappy with two different cost metrics.  Brighter paths are of lower cost.  Top: Penalizing path length. The first path is high cost (1866\,px), passing through both upper openings, and eventually converges to a path that passes both lower openings (1234\,px). Below: Penalizing low altitude paths. The first path passes through most of the lower openings (cost 1330), and the planner converges to a path that passes through upper openings (cost 321). }
\label{fig:flappy-costs}
\end{figure}

\section{Conclusion}

This paper presents an equivalence between optimal motion planning problems (either kinodynamic or kinematic) and feasible kinodynamic motion planning problems using a state-cost space transformation.  Despite the simplicity of the transformation, it is a powerful tool; we use it to develop an easily implemented, asymptotically-optimal, sampling-based meta-planner that accepts a sampling-based kinodynamic feasible planner as input.  It purely uses control-based sampling, making it suitable for problems with general differential constraints and cost functions that do not admit a steering function.  The expected convergence rate of the meta-planner is proven to be related to the goal-dependent running time of the underlying feasible planner.  Using RRT and EST as feasible planning subroutines, we demonstrate that the proposed method attains state-of-the-art performance on a number of benchmarks.

We hope this new formulation will provide inspiration and theoretical justification for new approaches to optimal motion planning.  As an example, an obvious way to improve convergence rate would be to run local optimizations on each trajectory found by the underlying planner; this method has been shown to work well for kinematic optimal path planning~\cite{anytime2013}.  We also obtained curious results regarding state-space vs cost-space weighting in the RRT distance metric.  Following up on this work may also open up avenues of research in sampling strategies for state-cost space planning, e.g., in appropriate biasing strategies.



\bibliographystyle{plainnat}
\bibliography{refs}

\end{document}